\documentclass{article}
\usepackage{graphicx}
\usepackage{hyperref} 
\usepackage{amsmath, amsthm, amssymb}
\usepackage{tikz}
\usepackage{natbib,multirow}
\usepackage{subcaption}
\usetikzlibrary{shapes.geometric}
\usetikzlibrary{angles,quotes}
\usetikzlibrary{calc} 
\usetikzlibrary{arrows.meta, positioning}
\usepackage{cleveref}

\usepackage{tcolorbox}
\tcbuselibrary{skins}
\usepackage{dashrule} 
\usepackage{xcolor}
\usepackage{wrapfig}
\usepackage{fullpage}

\usepackage{booktabs}
\usepackage{nicematrix}

\usepackage{array}
\usepackage{tabularx}
\usepackage{threeparttable}
\usepackage{listings}
\usepackage{float}

\lstdefinestyle{promptstyle}{
    basicstyle=\ttfamily\small,
    frame=single,
    framerule=0.4pt,
    rulecolor=\color{black!35},
    backgroundcolor=\color{black!3},
    breaklines=true,
    columns=fullflexible,
    keepspaces=true,
    xleftmargin=1em,
    xrightmargin=1em
}

\newtheorem{theorem}{Theorem}

\newtheorem{proposition}[theorem]{Proposition}

\newtheorem{definition}[theorem]{Definition}
\newtheorem{remark}[theorem]{Remark}

\newcommand{\righthand}{{RHI}}
\newcommand{\lefthand}{LHI}

\newcommand{\score}[1]{\(\sfrac{#1}{5}\)}

\usepackage{pifont}
\usepackage{tikz}
\usepackage{xcolor}

\definecolor{tabblue}{RGB}{31,119,180}

\definecolor{tab:green}{RGB}{44,160,44}
\definecolor{tab:red}{RGB}{214,39,40}

\newcommand{\succesmark}{%
  \textcolor{tab:green}{\ding{51}}%
}

\newcommand{\crossmark}{%
  \textcolor{tab:red}{\ding{55}}%
}

\renewcommand{\score}[1]{%
  \begin{tikzpicture}[baseline=-0.75ex]
    \def\scoreradius{0.95ex}%

    \foreach \i in {1,...,5}{%
      \pgfmathsetmacro{\startangle}{90-(\i-1)*72}%
      \pgfmathsetmacro{\endangle}{90-\i*72}%

      \ifnum\i>#1
        \def\sectorcolor{gray!12}%
      \else
        \def\sectorcolor{tabblue}%
      \fi

      \path[
        fill=\sectorcolor,
        draw=white,
        line width=0.25pt
      ]
      (0,0)
      -- (\startangle:\scoreradius)
      arc[
        start angle=\startangle,
        end angle=\endangle,
        radius=\scoreradius
      ]
      -- cycle;
    }%

    \draw[
      black!70,
      line width=0.3pt
    ] (0,0) circle[radius=\scoreradius];
  \end{tikzpicture}%
}

\usepackage{newfloat}
\usepackage{listings}
\DeclareCaptionStyle{ruled}{labelfont=normalfont,labelsep=colon,strut=off} 
\floatstyle{ruled}
\newfloat{listing}{tb}{lst}{}
\floatname{listing}{Listing}

\usepackage{booktabs}

\title{Indexing: the Beginning and the End}

\author{Alexander Kozachinskiy\\CENIA \\\texttt{alexander.kozachinskyi@cenia.cl} \and Vicente Opazo \\ CENIA \\\texttt{vicente.opazo@cenia.cl} \and Felipe Urrutia\\ Pontifical Catholic University of Chile
\\ \texttt{felipe.urrutia@uc.cl}}

\begin{document}
\maketitle

\begin{abstract}

We study information bottlenecks in modern deep-learning architectures -- RNNs, softmax transformers, linear-attention transformers and state-space models -- through the lens of the \emph{indexing primitive}. In this primitive, the input consists of $n$ bits and one integer $i$ from $1$ to $n$ called the index, and the output equals the value of the $i$-th bit.

We introduce causal complexity for masked architectures. We show that architectures with low causal complexity cannot solve the indexing primitive in any constant number of layers when the index appears at the end of the input. In particular, this limitation applies to low-parameter RNNs, SSMs and masked linear-attention transformers. In contrast, small softmax transformers can solve it in one layer, while non-masked linear-attention transformers can solve it in 2, which separates them from their masked counterparts. In turn, when the index appears at the beginning, we show that small RNNs are capable of solving this task in 1 layer, while all the other architectures require 2. 

All our impossibility results are unconditional and apply even to models that employ infinite-precision real arithmetic. Moreover, experiments for up to $n=64$ qualitatively align with our theory: configurations with low-parameter theoretical solutions learn the indexing task easily, while configurations that do not admit such theoretical solutions struggle to learn as the sequence length grows.


\end{abstract}

\section{Introduction}

Part of the research on the expressivity of transformers~\citep{strobl2024formal} has been revolving around their ability to compute certain fundamental \emph{primitives}. These primitives are supposed to be a mathematical abstraction of some challenges that arise in processing information by transformers. The most well-studied primitive is the \emph{function composition}.

It has been linked to the transformers' ability to infer implicit facts through \emph{compositional generalization}~\citep{peng2024limitations,guan2024mitigating}. For instance, the training data may contain the birth year of a singer $A$ and the Nobel Prize winners for each year, but not explicitly answer the question ``who won the Nobel Prize in Physics in the year $A$ was born?'' Answering it requires composing a function mapping people to birth years with one mapping years to Nobel Prize winners.

Mathematically, this inference is abstracted as the following task: given two functions $f,g\colon[n]\to[n]$ as lists of values, and an index $i\in[n]$, compute $g(f(i))$, where $[n]=\{1,\ldots,n\}$. \citet{peng2024limitations} introduced this task and used communication complexity to show that 1-layer transformers with $n^{o(1)}$ parameters and $n^{o(1)}$ bits of precision cannot solve it. They further showed that such transformers require $\Omega(k)$ chain-of-thought iterations for $k$-fold composition,
\[i\mapsto f_k\circ f_{k - 1} \circ \ldots \circ f_1(i).\]
Afterwards, \citet{kozachinskiy2026strassen} removed the bounded-precision assumption, proving the same impossibility for infinite-precision 1-layer transformers with $n^{o(1)}$ parameters. They have also introduced the \emph{binary relation composition} primitive and have conjectured that it cannot be done by $O(1)$-layer small transformers.
At the same time, \citet{barceloehrenfeucht} showed that $k$-fold composition requires exactly $k$ chain-of-thought steps for 1-layer transformers of any size, although only under unique hard attention.

In their breakthrough work, \citet{chen2025theoretical} were able to extend this analysis to multi-layer transformers. They have shown that $n^{o(1)}$-parameter $n^{o(1)}$-precision transformers \emph{with masking} require $k$ layers to solve a variation of the $k$-fold function composition task, but full-attention transformers can solve it in $O(\log k)$ layers. Their technique is based on ``autoregressive'' communication complexity and thus requires assumptions about bits of precision. It is open to obtain lower bounds on multi-layer infinite-precision transformers, even with masking.

Another primitive that arose from mechanistic interpretability is the \emph{induction heads} task. Given a sequence of symbols, the goal is to predict what will go after the last symbol. One simple strategy is -- take the closest previous appearance of the last symbol in the sequence and look what goes after it. It was observed empirically that this mechanism is widely used by transformer language models~\citep{elhage2021mathematical,olsson2022context}. In turn, \citet{sanford2024one} have shown that small 1-layer transformers cannot do this primitive while small 2-layer can (their notion of ``small'' includes the assumptions about the number of bits of precision).

Recently, \citet{strobl2025concise} have turned their attention  to an even more basic primitive of \emph{function evaluation} -- given a function $f\colon\{1,2,\ldots, n\}\to\{1,2,\ldots, n\}$ and an input $i$, output $f(i)$. They have found out that even this simple task might be hard for small 1-layer transformers under some assumptions about input representation. Their result applies even to infinite-precision transformers.

As we see, there is a rich literature on the analysis of these primitives for transformers. However, more modern architectures, aimed to circumvent the computational bottlenecks of transformers -- like Mamba~\citep{gu2023mamba} or linear-attention transformers~\citep{katharopoulos2020transformers} -- still wait their turn. In this paper, we address this gap by studying a version of function evaluation primitive called indexing -- which is essentially the function evaluation for functions that take just two values, 0 and 1. Our analysis includes standard transformers, linear-attention transformers, state-space models (SSMs), and ``pre-transformer'' RNNs.
Albeit very simple, the indexing primitive surprisingly  reveals a rich interplay between these architectures, and provides  insights on how they process information in a form of rigorous mathematical results. Moreover, it allows us to obtain  the first multi-layer infinite-precision lower bounds -- although not for transformers, but for a class of architectures with low \emph{causal complexity} -- including SSMs, RNNs, and masked linear-attention transformers. We now describe our results in more detail.
 
\paragraph{Our results}
In the indexing task, we are given $n$ bits $\sigma_1, \ldots, \sigma_n\in\{0, 1\}$ and a number $i\in\{1,2,\ldots, n\}$ called the index. The goal is to output $\sigma_i$. Observe that it can be seen as a function evaluation task for a function of the form $\sigma\colon\{1,2,\ldots, n\}\to \{0,1\}$.

We consider two versions of this task.  In the   left-hand indexing, the token with $i$ is the first one while the answer has to be computed in the special end-token. In the right-hand indexing, the token with $i$ is the last one, and the answer has to be computed in it. In both versions, there is also a token for each of the bits $\sigma_1, \ldots, \sigma_n$.

Our contribution is summarized as follows.

\begin{itemize}
    \item For causal architectures, or architectures with masking -- those where the output in $k$-th position does not depend on input in the positions after $k$ -- we introduce the notion of \emph{causal complexity} of a layer. Intuitively, it measures the minimal dimension to which the first $k - 1$ inputs to a layer can be compressed so that then the $k$-th output can be ``easily'' computed from this compressed representation  and the $k$-th input. 
    
    We show that no $O(1)$ layers with $n^{o(1)}$ causal complexity can solve the right-hand indexing task. We then observe that ``small'' ($n^{o(1)}$-parameter) masked linear-attention, Mamba and RNN layers all have $n^{o(1)}$ causal complexity, deriving the corresponding limitations for these architectures. Once again, our lower bound technique applies even in the infinite-precision regime. Thus, we conclude that no $O(1)$ small layers of these architectures, even if they use infinite-precision computations, can solve the right-hand indexing task.

    \item In contrast, as it essentially was observed in previous works~\citep{strobl2025concise} standard 1-layer transformers can solve the right-hand indexing task with $O(1)$ parameters. This is explained by their high causal complexity compared to other architectures. We also show that 2-layer non-masked linear-attention transformers with $O(1)$ parameters can solve the right-hand indexing task as well.  This separates them from their masked version which cannot solve the task in any constant number of $n^{o(1)}$-parameter layers.

    \item Finally, we observe that  in the left-hand indexing task, RNNs prevail over all other architectures. While RNNs can solve it in 1 layer with $O(1)$ parameters, neither standard attention, nor linear-attention, nor Mamba can solve it in 1 layer with $n^{o(1)}$ parameters (but all of them can solve it in 2 layers with $O(1)$ parameters, with or without masking).

\end{itemize}

Our results exploit information bottlenecks in the aforementioned architectures. In the right-hand indexing task, the bits are processed before the index. Thus, at the moment when index is considered, the information we receive from the previous positions must essentially allow us to recover the whole $n$-bit string. But such memory capacity turns out to be bounded by the causal complexity, which is low for every masked architecture in question except transformers. 

In turn, in the left-hand indexing task, the indexed $i$ is processed first, and intuitively it should be then carried over the  input until we reach the $i$-th position. Then the bit at this position should be remembered and carried over to the end. RNNs are perfectly suitable for this task, while, for instance, SSMs, which are sometimes viewed as parallelizable alternatives to RNNs because of their linear update rule, struggle on this task -- precisely, as we will show, because of that linearity.

The following table summarizes our results, indicating for each architecture the minimal number of ``small'' layers needed to solve the left-hand and the right-hand indexing tasks.

\begin{table}[!h]
\centering
\begin{tabular}{rcc}
\toprule
Architecture & LHI & RHI \\
\midrule
Full softmax  & 2 & 1 \\
Causal softmax & 2 & 1 \\
Full linear-attention & 2 & 2 \\
Causal linear-attention & 2 & $\omega(1)$  \\
SSM & 2 & $\omega(1)$  \\
RNN & 1 & $\omega(1)$  \\
\bottomrule
\end{tabular}
\caption{Minimum number of $n^{o(1)}$-parameter layers required for the left-hand
indexing (LHI) and the right-hand indexing (RHI). Recall that the $\omega(1)$ notation means ``is not bounded by a constant''.}
\label{tab:main-results}
\end{table}

All our impossibility results hold even for models that use infinite-precision computations with real numbers. We achieve this via adapting the VC dimension technique of \citet{kozachinskiy2026strassen} beyond the transformers context.

Finally, we experimentally test whether existence or absence of a low-parameter theoretical solution correlates with the  ability to learn the indexing task in practice. First, we find that already at length $n = 64$ the configurations on the lower-bound side consistently fail under our fixed training protocol. Across all seeds, causal linear-attention transformers, RNNs and Mamba (all with up to 4 layers) consistently fail to learn the right-hand indexing, as well as 1-layer transformers (softmax and linear) and 1-layer Mamba -- the left-hand indexing. As for the ``existence results'' (upper bounds), the situation seems to depend on the positional encoding. We found that polynomial positional encoding leads to the strongest agreement with our theory: among the 160 non-redundant constructive runs represented in our summary, only two failed to learn the task perfectly.

\subsection{Further related work} Study of the expressive power of transformers is not limited to the study of the above primitives. For instance, part of the research was dedicated to showing that these machines, equipped with the mechanism of autoregressively generating tokens, are capable of simulating any computation~\citep{perez2021attention,merrill2024expressive,jiang2025softmax}. A significant part of the literature is devoted to obtaining limitations for computing formal languages, using methods from computational complexity and logic~\citep{hahn2020theoretical,chiang2023tighter,yang2024masked,salzer2025counting}.

As for RNNs, long before the deep learning era it was observed that they are capable  of simulating finite automata~\citep{minsky1967computation}. Moreover, when one allows arbitrary rational weights and an unbounded number of intermediate steps after processing the input, they are capable of simulating Turing machines~\citep{siegelmann1994analog}. Very recently, \citet{merrill2026olmo} gave an example of a task, doable by a hybrid architecture (RNNs + transformers) but not by RNNs or transformers separately, under standard complexity assumptions.

In turn, SSMs have been investigated from the viewpoint of their ability to overcome limitations of transformers in recurrent computations~\citep{merrill2024illusion} -- as a more scalable alternative to RNNs in this matter. It was observed that this hugely depends on design choices. For instance, Mamba is not capable of performing summation over the group $\mathbb{Z}_2$~\citep{sarrof2024expressive}, while more general diagonalizable SSMs, under finite-precision assumption, are capable of performing summation over a finite group if and only if this group is solvable~\citep{shakerinava2026expressive}. 

\section{Preliminaries}
\subsection{VC dimension and arithmetic complexity}

\begin{definition}
\label{def_ar}
    Let $f\colon\mathbb{R}^n \to \mathbb{R}^m$. Its arithmetic complexity, denoted by $AC(f)$, is the minimal  $t\in\mathbb{N}$ such that $f$ can be computed by some algorithm in no more than $t$ of the following operations:
    \begin{itemize}
        \item the exponential function $\alpha\mapsto e^\alpha$ on real numbers;
        \item the arithmetic operations $+,-,\times,/$ on real numbers;
        \item jumps conditioned on $>, =$ comparisons of real numbers; 
        \item output a constant or a previously computed value.
    \end{itemize}
    (If no such $t$ exists, we set $AC(f) = +\infty$.)
    \end{definition}

By a concept class we mean a function $C\colon\mathbb{R}^n\times \mathbb{R}^p\to \{0, 1\}$. The first $n$ coordinates form an \emph{input}, and the last $p$ coordinates form \emph{parameters}.
\begin{definition}
    Let $C\colon\mathbb{R}^n\times\mathbb{R}^p \to \{0, 1\}$ be a concept class. Its VC dimension, denoted by $VCdim(C)$, is the maximal natural number $d$ such that for some $x^1, \ldots, x^d\in\mathbb{R}^n$, the following holds. For every $d$-bit string $\alpha = \alpha_1\ldots \alpha_d\in\{0,1\}^d$ there exists $y\in\mathbb{R}^p$ such that:
    \[C(x^1, y) = \alpha_1, \ldots, C(x^d, y) = \alpha_d.\]
\end{definition}

\begin{theorem}[\citep{anthony2009neural} Theorem 8.14]
\label{thm_vc}
Let $C$ be a concept class with $p$ parameters. Then
\[VCdim(C) = O(AC(C)^2 \cdot p^2). \]
\end{theorem}

\subsection{Layers and tasks}
We work with a notion of a \emph{layer}. In general, a layer with input length $n$ and input-output dimension $d$ (in short, an $(n,d)$-layer) is any function $f\colon(\mathbb{R}^d)^n \to (\mathbb{R}^d)^n$.

By \emph{tasks} we formally mean functions $T\colon\Sigma_1 \times \ldots \times \Sigma_n \to \Gamma$, where $\Sigma_1, \ldots, \Sigma_n,\Gamma$ are finite sets. Here $n$ is the input length of a task. We imagine that the input to a task is given in a sequence of $n$ \emph{tokens}: the first token contains an element $\sigma_1\in \Sigma_1$, the second token contains an element $\sigma_2\in \Sigma_2$, and so on. The goal is to compute $T(\sigma_1\ldots \sigma_n)$.

We now formalize what it means that $L$ layers $f_1, \ldots, f_L$ solve a task $T$.

\begin{definition}
  Let $T \colon\Sigma_1\times \ldots \times \Sigma_n \to \Gamma$ be a task, and $f_1, \ldots, f_L$ be $(n,d)$-layers.

  We say that $f_1, \ldots, f_L$ solve $T$ if there exists a function
  \[p\colon\Sigma_1\times\{1\}\cup\ldots\cup \Sigma_n\times\{n\} \to \mathbb{R}^d,\]
(usually referred to as input-position embedding), and a matrix $A\in\mathbb{R}^{\Gamma \times d}$ (usually referred to as output-distribution matrix) such that for any $\sigma_1\ldots\sigma_n\in \Sigma_1\times \ldots \times\Sigma_n$, the following holds. Calculate
\[x_1 = p(\sigma_1, 1), \ldots, x_n = p(\sigma_n, n),\]
\[(y_1, \ldots, y_n) = f_L\circ \ldots \circ f_1((x_1, \ldots, x_n)),\]
\[\alpha = A y_n\in\mathbb{R}^\Gamma.\]
Then, denoting $\gamma = T(\sigma_1\ldots\sigma_n)$, we have $\alpha_\gamma > \alpha_{\gamma'}$ for all $\gamma'\in\Gamma\setminus\{\gamma\}$.
\end{definition}

\paragraph{Indexing task.} The indexing task is the following task: given $n$ bits $\sigma_1, \ldots, \sigma_n\in\{0, 1\}$ and a number $i\in\{1, 2,\ldots, n\}$, the goal is to output $\sigma_i$. We consider two variations -- the \emph{right-hand indexing task} where the input is given via $n +1$ tokens; the last one (where the output should be computed) has $i$.  The left-hand indexing task assumes $n + 2$ tokens, first goes the token with $i$, then tokens with $\sigma_1, \ldots, \sigma_n$, and then a special end token.

\subsection{Attention, RNNs, SSMs} We define several types of layers  from relevant architectures of interest. For each of them, we define the notion of \emph{size}, referring to the number of learnable parameters in their MLPs (that, unless stated otherwise, are assumed to use the ReLU activation). We also take into account arithmetic complexity of some non-polynomial components of the architectures. The size  always includes a term $d$ -- the dimension of the layer -- and thus upper bounds the dimension.

\begin{definition}
    A \textbf{causal softmax transformer} $(n,d)$-layer is given by 3 matrices $K, Q, V \in\mathbb{R}^{d\times d}$ and an MLP $\mathcal{N}\colon\mathbb{R}^d \to\mathbb{R}^d$.

    On input $(x_1, \ldots, x_n)\in(\mathbb{R}^d)^n$, the following computations are performed to produce the output sequence $(y_1, \ldots, y_n)\in(\mathbb{R}^d)^n$. First,
   \[v_\ell = V x_\ell, \qquad q_\ell = Q x_\ell, \qquad k_\ell = K x_\ell\]
   for $\ell = 1, \ldots, n$. Then
    \begin{equation}
    \label{eq_attention}
   a_k = \frac{\sum\limits_{\ell = 1}^k v_\ell  \exp\{k_\ell^T q_k\} }{\sum\limits_{\ell = 1}^k \exp\{k_\ell^T q_k\}},\qquad k = 1, \ldots, n,     
    \end{equation}
    and then
    \[y_k = \mathcal{N}(x_k + a_k), \qquad k = 1, \ldots, n\]
    In a \textbf{full softmax transformer} layer, the upper limits in the sums in \eqref{eq_attention} are changed from $k$ to $n$.

    The size of the layer is $3d^2$ plus the number of learnable parameters of $\mathcal{N}$.
\end{definition}

\begin{definition}
    A \textbf{causal linear-attention} $(n,d)$-layer is given by 3 matrices $K, Q, V \in\mathbb{R}^{d\times d}$, a kernel function $\phi\colon\mathbb{R}^d \to (0,+\infty)^d$ and an MLP $\mathcal{N}\colon\mathbb{R}^d \to\mathbb{R}^d$.

    On input $(x_1, \ldots, x_n)\in(\mathbb{R}^d)^n$, the following computations are performed to produce the output sequence $(y_1, \ldots, y_n)\in(\mathbb{R}^d)^n$. First,
   \[v_\ell = V x_\ell, \qquad q_\ell = Q x_\ell, \qquad k_\ell = K x_\ell\]
   for $\ell = 1, \ldots, n$. Then
      \begin{align}
    \label{eq_lin_attention1}
   a_k &= \frac{\sum\limits_{\ell = 1}^k v_\ell  \cdot \phi(k_\ell)^T \phi(q_k) }{\sum\limits_{\ell = 1}^k \phi(k_\ell)^T \phi(q_k)},\qquad k = 1, \ldots, n, \\
       \label{eq_lin_attention2}
   y_k &= \mathcal{N}(x_k + a_k), \qquad k = 1, \ldots, n.
    \end{align}
    In a \textbf{full linear-attention} layer, the upper limit in the sums in \eqref{eq_lin_attention1} are changed from $k$ to $n$.

    The size of the layer is $3d^2$ plus the number of learnable parameters of $\mathcal{N}$, plus the arithmetic complexity of the kernel function $\phi$. 
\end{definition}

For instance, in \citep{katharopoulos2020transformers} the following kernel function is used:
\[\phi\colon (x_1, \ldots, x_d) \mapsto (elu(x_1) + 1, \ldots, elu(x_d) + 1), \]
where
\[elu(x) = \begin{cases}x & x > 0 \\ \alpha (e^x -1) & x\le 0\end{cases}\]
for some positive constant $\alpha$. This example has $O(d)$ arithmetic complexity.
Another common choice is to use ReLU instead of elu, where the arithmetic complexity is still $O(d)$.

\begin{remark}
\label{rem:one-layer-mask-equivalence}
    For a one-layer softmax or linear-attention transformer, full and causal attention induce the same computation at the final readout token. Indeed, the answer is computed from the last token which sees all tokens in both models, and does not yet use the attention computations in these tokens.
\end{remark}

\begin{definition}
    An SSM layer of dimension $d$ is given by three maps $A\colon\mathbb{R}^{d}\to\mathbb{R}^{d\times d}, B\colon \mathbb{R}^{d}\to\mathbb{R}^d, \phi\colon\mathbb{R}^{2d} \to\mathbb{R}^d$ and vector $h_0\in\mathbb{R}^d$. On input $(x_1, \ldots, x_n)\in(\mathbb{R}^d)^n$ the output sequence of vectors $(y_1, \ldots, y_n)\in(\mathbb{R}^d)^n$ is computed as follows: for $k = 1, \ldots, n$, compute
      \begin{align}
    \label{mamba1}
        h_k &= A(x_k) h_{k - 1} + B(x_k),\\
          \label{mamba2}
        y_k &= \phi(x_k, h_k).
    \end{align}
    The size of the layer is $d^2 + 2d$  plus the arithmetic complexities of $A,B,\phi$.
\end{definition}
For instance, in Mamba \citep{gu2023mamba}, $A(x_t)$ is a constant $d\times d$ matrix, and $B(x_t)$ is defined as a multiplication of the vector $x_t$ by another constant $d\times d$ matrix. Thus, these maps have $O(d^2)$ arithmetic complexity. In turn, $\phi$ in Mamba can be seen as an MLP with the SiLU/Swish activation~\citep{ramachandran2018searching}, which can be expressed using exponentials and standard arithmetic operations. Thus, such $\phi$ also has $poly(d)$ arithmetic complexity. 

Additionally, Mamba has an optional normalization layer for $\phi$ using LayerNorm of \citet{ba2016layer}. Unfortunately, such a layer cannot be easily handled in our theoretical setting  due to its use of the square root function which is not an operation permitted in Definition \ref{def_ar}. Thus, with layernorm, we no longer have good upper bounds on the VC dimension guaranteed by Theorem \ref{thm_vc} that we require for our proofs. One can incorporate this type of layers with the use of bounds on the VC dimension for Pfaffian activation functions due to~\citet{karpinski1997polynomial}. However, the way these results are formulated requires a significant additional technical work in order to apply to our setting which we omit in the current version for simplicity.

\begin{definition}
    An RNN layer of dimension $d$ is given by two MLPs $\mathcal{N}_1, \mathcal{N}_2\colon\mathbb{R}^{2d}\to\mathbb{R}^d$. On input $(x_1,\ldots, x_n)\in(\mathbb{R}^d)^n$, the output sequence of vectors $(y_1, \ldots, y_n)\in(\mathbb{R}^d)^n$ is computed as follows: set $h_0=0\in\mathbb{R}^d$ and for $k = 1, \ldots, n$, compute
    \begin{align}
    \label{rnn1}
        h_k &= \mathcal{N}_1(x_k, h_{k-1}),\\
          \label{rnn2}
        y_k &= \mathcal{N}_2(x_k, h_k).
    \end{align}
    The size of the layer is $d$ plus the number of learnable parameters of $\mathcal{N}_1, \mathcal{N}_2$.
\end{definition}

\section{Right-hand Indexing and Causal Architectures}

\begin{definition}
    An $(n,d)$-layer $f$ is \textbf{causal} if the $ i$-th output of the layer is the function of the first $i$ input vectors, for every $i = 1, \ldots, n$.
\end{definition}
\begin{definition}
    Let $f$ be an $(n,d)$-causal layer. Its \textbf{causal complexity} is the minimal natural number $s\ge d$ such that for some $p, t\in\mathbb{N}$ with $p + t \le s$, the following holds.

    For any $k  = 2, \ldots, n$, there exist a ``fingerprint'' function
    \[\phi_k:(\mathbb{R}^d)^{k-1} \to \mathbb{R}^p\]
    such that the $k$-th output vector, $y_k$, on input $(x_1, \ldots, x_n)\in(\mathbb{R}^d)^n$, is the function of the fingerprint $f_k = \phi_k(x_1, \ldots, x_{k-1})$ and of $x_k$,
    and, moreover, this function $(f_k, x_k) \mapsto y_k$ has arithmetic complexity at most $t$.
    
\end{definition}

\begin{theorem}
\label{thm_rh}
    \begin{itemize}
    \item[a)] No $O(1)$ causal layers of causal complexity $n^{o(1)}$ can solve the right-hand indexing task. 

    \item[b)] Two full linear-attention layers of $O(1)$-size can solve the right-hand indexing task;

    \item[c)] One softmax transformer layer of $O(1)$-size can solve the right-hand indexing task.
    \end{itemize}
\end{theorem}
\begin{proof}
\begin{itemize}
    \item[a)] Assume for contradiction that there exist $L = O(1)$ causal layers $f_1, \ldots, f_L$ of causal complexity $n^{o(1)}$ that solve the right-hand indexing task
    \[(\sigma_1, \ldots, \sigma_n, i)\mapsto \sigma_i.\]

    We employ the definition of the causal complexity with respect to the $(n+1)$-st token, one with the index $i$. First, this input token is transformed via some input embedding into a vector $x_{n+1}(i)\in\mathbb{R}^d$. Next, the input binary word $\sigma = \sigma_1\ldots \sigma_n\in\{0, 1\}^n$ is first transformed into a sequence $(x_1, \ldots, x_n)$ of input vectors in the first $n$ tokens. Then $(x_1, \ldots, x_n)$, by definition of causal complexity, can be transformed in $L$ fingerprints of dimension $p = n^{o(1)}$:
    \[\sigma\mapsto f_{n+1}(\sigma)\in\mathbb{R}^{O(p)}\]
    such that the outputs of all $L$ layers in the last position, including the last output $y_{n+1}^L$, can be computed from $f_{n+1}(\sigma)$ and $x_{n+1}(i)$ in arithmetic complexity $n^{o(1)}$. Then from $y_{n+1}^L$ one can compute the answer $\sigma_i$ to the indexing task through  multiplying $y_{n+1}^L$ by a $2\times d$ matrix, and computing the maximal coordinate of the resulting vector. The last part requires $O(d) = n^{o(1)}$ arithmetic complexity (recall that causal complexity upper bounds $d$).

One can perform these computations for arbitrary real vectors $f_{n+1}\in\mathbb{R}^{O(p)}$ and $x_{n+1}\in\mathbb{R}^d$, not just from those that come as inputs to the indexing problem. That is, when we omit inputs in $f_{n+1}(\sigma),x_{n+1}(i)$, we no longer consider these vectors as functions of $\sigma$ and $i$ but allow them to take arbitrary values. This gives a concept class $C(x_{n+1}, f_{n+1})$, which, on the one hand, has arithmetic complexity, and thus VC dimension by Theorem \ref{thm_vc}, at most $n^{o(1)}$, but on the other hand, has VC dimension at least $n$ since for inputs $x_{n+1}(1), \ldots, x_{n+1}(n)$ we have:
\[C(x_{n+1}(i),f_{n+1}(\sigma)) = \sigma_i, \qquad i = 1, \ldots, n,\]
for every $\sigma\in\{0, 1\}^n$.

\item[b)] Using a positional encoding that assigns 0 to every position except the last one, to which it assigns a sufficiently large value, we can make the last token receive attention weight $100n^2$, while all other tokens receive weight $1$. Since there is no masking, every token can then recover the index $i$ from the last token up to a small additive error, say $0.01$. Next, we can devise an output MLP of the first layer that produces 0 at positions $k\neq i$, while at position $i$ it outputs a large marker together with the local bit $\sigma_i$. Indeed, one can do it by computing the absolute value $|k - i'|$ -- if it is less than $0.01$, then $k = i$, and if it is greater than $0.99$, then $k\neq i$. Thus, the first-layer output at position $i$ contains a large marker together with $\sigma_i$, while all other positions output $0$. Repeating the same attention mechanism in the second layer gives this value to every token, including the last one.

\item[c)] Imagine that the positional encoding of the $k$-th position contains $\begin{pmatrix}\cos k \\ \sin k\end{pmatrix}$ while the index $i$ is embedded as $\begin{pmatrix}\cos i \\ \sin i\end{pmatrix}$. 
Then we can make sure that the attention from the index position to the $k$-th position is of the form
\[\left\langle\begin{pmatrix}\cos k \\ \sin k\end{pmatrix},  \begin{pmatrix}\cos i \\ \sin i\end{pmatrix}\right\rangle \]
and thus is strictly maximized at $k = i$. Through multiplying these products by a sufficiently large constant, the softmax function will get us, up to a small error, the content of the $i$-th position that allows us to restore the $i$-th bit.
\end{itemize}
    
\end{proof} 


We now observe that item a) of Theorem \ref{thm_rh} applies to RNNs, SSMs and causal linear-attention layers with reasonable assumptions on their size.

\begin{proposition} $n^{o(1)}$-size SSM, RNN, and causal linear-attention layers 
have causal complexity $n^{o(1)}$.
\end{proposition}
\begin{proof} Let $(x_1, \ldots, x_n)\in(\mathbb{R}^d)^n$ be an input. 
    For SSMs and RNNs, the fingerprint of $x_1, \ldots, x_{k - 1}$ would be the vector $h_{k - 1}$ from \eqref{mamba1} and \eqref{rnn1}, respectively.
    Given this vector and $x_k$, one computes $y_k$ according to formulas (\ref{mamba1}--\ref{mamba2}) and  (\ref{rnn1}--\ref{rnn2}). This can be performed in arithmetic complexity $n^{o(1)}$.  Here one uses the fact that the dimension, the number of learnable parameters of the MLPs in the equations, and the arithmetic complexity of $A,B,\phi$ in case of SSMs, are bounded by $n^{o(1)}$ due to the corresponding bound on the size of a layer.

    For causal linear-attention layer, the fingerprint of $x_1, \ldots, x_{k - 1}$ consists of one $d\times d$ matrix, and one $d$-dimensional row, computed by:
    \[\sum\limits_{\ell = 1}^{k -1} v_\ell \phi(k_\ell)^T,\qquad \sum\limits_{\ell = 1}^{k -1}  \phi(k_\ell)^T.\]
The output $y_k$ is then computed via formulas in (\ref{eq_lin_attention1}--\ref{eq_lin_attention2}). Complementing the computation, given two sums above, requires now just $d^{O(1)} = n^{o(1)}$ arithmetic operations that come from matrix products, and also computation of $\phi$ whose arithmetic complexity is bounded by $n^{o(1)}$ due to the bound on the size of the layer.
    
\end{proof}


\section{Left-hand Indexing: RNNs prevail}
\begin{theorem}
\label{thm_left}
\begin{itemize}
       
        \item[a)] One $O(1)$-size RNN layer can solve the left-hand indexing task. 

        \item[b)] No 1-layer $n^{o(1)}$-size SSM, linear-attention or softmax transformer  can solve the left-hand indexing task.

        \item[c)] Each of the following -- 2 $O(1)$-size SSM layers, 2 $O(1)$-size linear-attention layers, 2 $O(1)$-size softmax transformer layers -- can solve the left-hand indexing task. This holds regardless of whether the transformer layers are full or causal.
\end{itemize}
\end{theorem}
\begin{proof}
    \begin{itemize}
        \item[a)] The RNN stores $i$ and one answer coordinate. At bit position $k$, the update MLP computes a binary marker for $i=k$ and writes $\sigma_k$ to the answer coordinate only at that position. Since exactly one position is selected, the final state contains $\sigma_i$.

        \item[b)] We start with the lower bound for transformers (the same proof works both for linear-attention and softmax). Let the input to the indexing task be $i, (\sigma_1, \ldots, \sigma_n)$. There are $n +2$ tokens, the first token with $i$, next $n$ tokens with input bits, and the last end-token. Note that for both models, the computation of the attention in the last token  (see \eqref{eq_attention} and \eqref{eq_lin_attention1}) can be decomposed into two parts, one determined by $i$ and the other by the string $\sigma = \sigma_1 \ldots \sigma_n$, more precisely:
        \begin{equation}
        \label{a_n}
            a_{n + 2} = \frac{\alpha(i) + \beta(\sigma)}{\gamma(i) + \rho(\sigma)},
       \end{equation}
where $\alpha(i),\beta(\sigma) \in\mathbb{R}^d$, $\gamma(i), \rho(\sigma)\in(0,+\infty)$. 
Consider a concept class $C$ that treats coordinates of $(\alpha(i), \gamma(i))$ as $d + 1$ input coordinates and $(\beta(\sigma), \rho(\sigma))$
 as $d + 1$ parameters, computes $a_{n+2}$ as in \eqref{a_n}, and then performs the rest of the computations in the last token, getting the output of the indexing task. More specifically, it sums up $a_{n+2}$ with the vector $x_{n +2}$ -- embedding of the last token, applies the output MLP $\mathcal{N}$ to the sum, multiplies the resulting vector by the output-distribution matrix $A$, and computes which output class has the largest score. On the one hand, if the size of the layer is $n^{o(1)}$, then $d = n^{o(1)}$, the number of learnable parameters in $\mathcal{N}$ is $n^{o(1)}$, and hence the number of parameters and arithmetic complexity of $C$ is $n^{o(1)}$, together with its VC dimension by Theorem \ref{thm_vc}. On the other hand, if for contradiction such a layer solves the indexing task, then
 \[C\big((\alpha(i),\gamma(i)),(\beta(\sigma),\rho(\sigma))\big)=\sigma_i,\]
 meaning that the VC dimension of $C$ is at least $n$ as, choosing distinct $\sigma\in\{0, 1\}^n$, we can realize any Boolean function on $n$ inputs:
 \[(\alpha(1), \gamma(1)), \ldots, (\alpha(n), \gamma(n)).\]

 The structure of the  argument against $n^{o(1)}$-size SSM layer, computing the left-hand indexing, is similar. Assuming for contradiction that such a layer exists, we construct a concept class 
$C(\alpha, \beta)$ with $n^{o(1)}$ number of parameters and $n^{o(1)}$ arithmetic complexity, satisfying:
\[C(\alpha(i), \beta(\sigma)) = \sigma_i\]
for some functions $\alpha(i), \beta(\sigma)$, obtaining a contradiction since its VC dimension has to be at least $n$.  More specifically, we let $\alpha(i) = h_1$ be the state after processing the first token (one with the index) in (\ref{mamba1}--\ref{mamba2}). The further evolution of the state $h_1$ happens through a composition of $n + 1$ affine functions $\mathbb{R}^d \to \mathbb{R}^d$, determined by $\sigma_1, \ldots, \sigma_n$. This whole composition can be given by a $d\times d$ matrix $A(\sigma)$ and a $d$-vector $b(\sigma)$. Overall, we get the following formula for the state after processing the whole input
\[h_{n+2} = A(\sigma) \alpha(i) + b(\sigma).\]
Coordinates of $\beta(\sigma) = (A(\sigma), b(\sigma))$ will be viewed as parameters, and there are $d^2 +d = n^{o(1)}$ of them. Computing $h_{n+2}$ takes $poly(d) = n^{o(1)}$ standard arithmetic operations.
The output of the SSM layer is then computed by applying the output MLP $\phi$ to $h_{n+2}$ and the embedding of the last token $x_{n+2}$, multiplying by the output-distribution matrix, and taking the highest-scoring output class. Again, this all takes arithmetic complexity $n^{o(1)}$ if the size of the layer is $n^{o(1)}$. Finally, we have $C(\alpha(i), \beta(\sigma)) = \sigma_i$ since the layer is supposed to solve the indexing task.

        \item[c)] Transformers -- both linear-attention and softmax, causal and full -- essentially can mimic the solution from item b) of Theorem \ref{thm_rh}. By putting a lot of weight to the first position in the positional encoding, all subsequent tokens can obtain the value of the index $i$ (up to a small additive error, say, $0.01$). Since the index is in the beginning, the masking does not prevent us from doing it. Then again, the output MLP of the first layer can use this to output a large number at the $i$-th position, and 0 at all other positions. Finally, the second layer can now focus its attention on the $i$-th position, retrieving the value of the $i$-th bit. Since the output is computed at the last token, masking would not be a problem. 

        A 2-layer SSM is capable of imitating this solution as well. Note that due to its linear update rule, it can in particular simply sum up the input vectors. By putting the index $i$ to a designated coordinate for the first token, and 0 for other tokens, and summing up values in this coordinate, we get $i$ to every position after the first layer. Then the output MLP can transform it into the value of the $i$-th bit at the $i$-th position and 0 at the rest of the positions. Summing up these MLP outputs at the second layer gives us the $i$-th bit.
    \end{itemize}
\end{proof}

\section{Experiments}

The preceding sections establish asymptotic separations for indexing. We now test whether the same picture appears in finite models trained from scratch. We call configurations supported by our positive constructions \emph{constructive} (\succesmark), and configurations on the asymptotically hard side of our results \emph{stress tests} (\crossmark). Stress-test configurations may still succeed at small finite lengths, but their performance is expected to deteriorate as $n$ grows. These experiments do not replace the lower bounds: training failure is not evidence of impossibility.

The appendix provides full hyperparameters, ablations, parameter counts, runtimes, and length-shifted evaluations. The code for our experiments can be found at \[\text{\href{https://github.com/visho33/IndexRetrieval}{\includegraphics[height=1em]{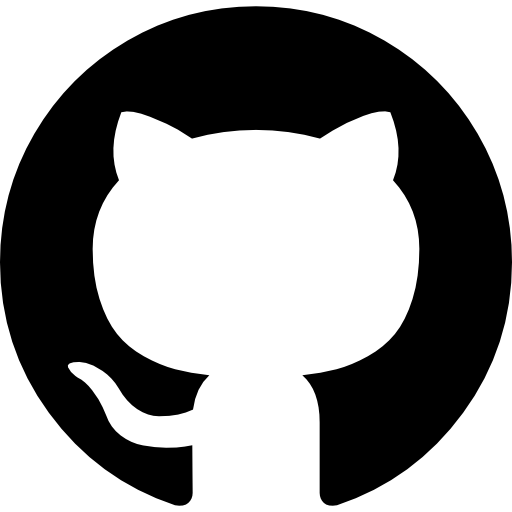} \texttt{https://github.com/visho33/IndexRetrieval}}}.\]

\subsection{Experimental Setup}

We train models on binary indexing for $n\in\{8,16,32,64\}$, using five random seeds per configuration and online-generated examples. We use the query-token readout for RHI and the end-token readout for LHI.

We compare full and causal variants of softmax and linear attention, together with GRUs and Mamba2. All models have width 16. Attention models use two heads and feedforward width 64, and linear attention uses the ELU+1 feature map. By Remark~\ref{rem:one-layer-mask-equivalence}, equivalent one-layer full/causal pairs are reported once. Further implementation details are provided in the supplement.

Following an ablation over learned, sinusoidal, rotary, polynomial, and no positional encoding, we use polynomial positional encodings throughout the main sweep. Bit tokens encode their type and value, whereas query tokens additionally contain the features $((i-1)/n)$ and $((i-1)/n)^2$.

Models are trained with AdamW, learning rate $10^{-3}$, batch size 256, and at most 500 epochs, using 50,000 newly generated examples per epoch. A run is considered successful if its maximum held-out accuracy during training reaches $1.0$.

\subsection{Results}

Table~\ref{tab:expected-outcomes-and-results} reports the number of successful seeds for each of the 23 distinct configurations. Constructive configurations succeed in $158/160$ runs overall: all $40/40$ runs succeed for $n\leq32$, and $38/40$ succeed at $n=64$. The only failures are one seed of two-layer full softmax attention and one seed of two-layer Mamba, both on LHI. Stress-test configurations display the opposite trend. Their success count decreases from $70/75$ at $n=8$, to $51/75$ at $n=16$, $13/75$ at $n=32$, and $0/75$ at $n=64$. Thus, the empirical separation between the two theory-guided groups becomes increasingly sharp as the sequence length grows.

For right-hand indexing, all stress-test families succeed at $n=8$, while additional depth allows some causal linear-attention and Mamba models to remain successful at $n=32$. Nevertheless, every causal linear-attention, GRU, and Mamba configuration fails the criterion at $n=64$, for all tested depths. One-layer full and causal linear attention also transition from $5/5$ successes at $n\leq16$ to $0/5$ at $n\geq32$.

For left-hand indexing, the breakdown occurs earlier. One-layer softmax attention fails at every tested length, whereas one-layer linear attention and Mamba succeed at $n=8$ but fail for every seed once $n\geq16$. In contrast, the constructive one-layer GRU and two-layer models succeed almost uniformly.

Overall, the experiments qualitatively align with the theoretical
separations while exhibiting substantial finite-size effects. Success at small $n$ does not contradict an asymptotic lower bound, and failure under our fixed optimization protocol is not evidence of impossibility. Several unsuccessful runs at $n=64$ also remain above chance, as detailed in the supplementary material.

\begin{table}[!t] 
\centering 
\begin{tabularx}{0.6\columnwidth}{clcccccc} 
\toprule
\multirow{2}{*}{Task}
& \multirow{2}{*}{Model family}
& \multirow{2}{*}{$L$}
& \multirow{2}{*}{Thm.}
& \multicolumn{4}{c}{Bit string length \(n\)} \\
\cmidrule(lr){5-8}
& & & & 8 & 16 & 32 & 64 \\
\midrule 
\righthand & Softmax & 1 & \succesmark
& \score{5} & \score{5} & \score{5} & \score{5} \\

\righthand & Linear & 1 & \crossmark
& \score{5} & \score{5} & \score{0} & \score{0} \\

\righthand & Full linear & 2 & \succesmark
& \score{5} & \score{5} & \score{5} & \score{5} \\

\righthand & Causal linear & 2 & \crossmark
& \score{5} & \score{5} & \score{2} & \score{0} \\

\righthand & Causal linear & 3 & \crossmark
& \score{5} & \score{5} & \score{3} & \score{0} \\

\righthand & Causal linear & 4 & \crossmark
& \score{5} & \score{5} & \score{4} & \score{0} \\

\righthand & GRU & 1 & \crossmark
& \score{5} & \score{0} & \score{0} & \score{0} \\

\righthand & GRU & 2 & \crossmark
& \score{5} & \score{3} & \score{0} & \score{0} \\

\righthand & GRU & 3 & \crossmark
& \score{5} & \score{5} & \score{0} & \score{0} \\

\righthand & GRU & 4 & \crossmark
& \score{5} & \score{4} & \score{0} & \score{0} \\

\righthand & Mamba & 1 & \crossmark
& \score{5} & \score{4} & \score{0} & \score{0} \\

\righthand & Mamba & 2 & \crossmark
& \score{5} & \score{5} & \score{0} & \score{0} \\

\righthand & Mamba & 3 & \crossmark
& \score{5} & \score{5} & \score{1} & \score{0} \\

\righthand & Mamba & 4 & \crossmark
& \score{5} & \score{5} & \score{3} & \score{0} \\

\lefthand & GRU & 1 & \succesmark
& \score{5} & \score{5} & \score{5} & \score{5} \\

\lefthand & Mamba & 1 & \crossmark
& \score{5} & \score{0} & \score{0} & \score{0} \\

\lefthand & Mamba & 2 & \succesmark
& \score{5} & \score{5} & \score{5} & \score{4} \\

\lefthand & Softmax & 1 & \crossmark
& \score{0} & \score{0} & \score{0} & \score{0} \\

\lefthand & Full softmax & 2 & \succesmark
& \score{5} & \score{5} & \score{5} & \score{4} \\

\lefthand & Causal softmax & 2 & \succesmark
& \score{5} & \score{5} & \score{5} & \score{5} \\

\lefthand & Linear & 1 & \crossmark
& \score{5} & \score{0} & \score{0} & \score{0} \\

\lefthand & Full linear & 2 & \succesmark
& \score{5} & \score{5} & \score{5} & \score{5} \\

\lefthand & Causal linear & 2 & \succesmark
& \score{5} & \score{5} & \score{5} & \score{5} \\

\midrule 
\multicolumn{3}{c}{\multirow{2}{*}{\textit{Total successful runs}}}
& \succesmark & 40 & 40 & 40 & 38 \\

& & & \crossmark & 70 & 51 & 13 & 0 \\ 
\bottomrule 
\end{tabularx} 

\caption{
Successful seeds out of five by task, model, depth $L$, and bit-string length $n$. For $L=1$, Softmax and Linear merge the equivalent full and causal variants. \protect\succesmark{} and \protect\crossmark{} denote constructive and stress-test configurations. The final rows aggregate 40 constructive and 75 stress-test runs per length.
}
\label{tab:expected-outcomes-and-results} 
\end{table}

\section{Conclusion}

We used indexing to isolate how token order and information bottlenecks affect sequence models. Causal complexity yields precision-independent lower bounds for constant-depth RNNs, SSMs, and causal linear-attention transformers on right-hand indexing, whereas left-hand indexing gives a one-layer advantage to RNNs. Our experiments are consistent with these theoretical separations and show that their signatures are already visible at moderate input lengths. They also reveal substantial finite-size effects: success at small $n$ does not contradict an asymptotic lower bound, and training failure is not evidence of impossibility. More broadly, our results show that even a minimal retrieval primitive can expose fundamental differences in how sequence architectures route and preserve information. Extending causal complexity to richer primitives and more general architectural components is a direction for future work.

\appendix

\section{Further Experimental Details}

This supplement provides the experimental details omitted from the main paper: the full training and evaluation protocol, the positional-encoding ablation used to choose the final setup, complete held-out results for each model family, parameter counts, runtimes, and out-of-distribution evaluations.

The main paper reports 23 distinct configuration classes. For one-layer attention models, the full and causal variants are theoretically equivalent at the final readout token, so they are merged there. In this supplement, we sometimes report them separately because both variants were trained independently. Whenever this affects a total, we state whether we are counting merged configuration classes or individual trained variants.

All experiments use the binary indexing task defined in the main paper. Each run is evaluated once per epoch on a fixed held-out set of 2,000 examples. Training examples are generated online, with 50,000 examples per epoch. Training stops when the held-out binary cross-entropy reaches at most $10^{-6}$ or after 500 epochs.

A run is considered successful if its maximum held-out accuracy during training reaches $1.0$. Since indexing has a deterministic correct answer, requiring perfect held-out accuracy gives a simple and strict empirical criterion. However, success on this finite held-out set does not prove that the trained model solves every possible input.

\subsection{Protocol Details}

Table~\ref{tab:supp-experimental-setup} lists the final training and evaluation settings used in the main sweep. The positional-encoding alternatives considered during development are described separately in the next section.

For each example, the index $i$ is sampled uniformly from $\{1,\ldots,n\}$, and the input bits are sampled independently and uniformly from $\{0,1\}$. The target is $\sigma_i$. Training examples are generated online, while the held-out set is generated once and kept fixed throughout each run. For a run with seed $s$, we initialize Python, NumPy, PyTorch, and the CUDA random-number generators with $s$. The held-out set uses seed $s+10{,}000$, while the length-shifted sets at $n/2$ and $2n$ use seeds $s+20{,}000$ and $s+30{,}000$, respectively.

Softmax attention is implemented using Hugging Face BERT blocks (except rotary, which uses RoFormer). We explicitly set the model width, number of layers, number of attention heads, feed-forward width, causal or full attention mask, and zero dropout. Other internal choices follow the defaults of the installed Transformers version.

GRUs use PyTorch's \texttt{nn.GRU} with input and hidden size 16, one direction, and zero dropout. Biases and the initial hidden state follow the PyTorch defaults.

Mamba2 uses state dimension 16, head dimension 8, expansion factor 1, and convolution width 3. Caching, fused normalization, and fused cross-entropy paths are disabled. Other unspecified Mamba2 options follow the defaults of the installed \texttt{flash-linear-attention} version. We refer to this model family as \emph{Mamba} in the result tables.

Linear attention is implemented locally using pre-normalized blocks, bias-free query, key, value, and output projections, the $\operatorname{ELU}(x)+1$ feature map, denominator epsilon $10^{-6}$, and zero dropout. Full linear attention uses global key--value sums, whereas causal linear attention uses prefix sums. Other unspecified options follow the PyTorch defaults.

\newpage

\begin{table}[!h]
\centering
\begin{tabular}{ll}
\toprule
Setting & Value \\
\midrule
Task sizes & $n\in \{8,16,32,64\}$ \\
Seeds & $\{0,1,2,3,4\}$ \\
Training examples & 50000 per epoch, drawn on demand \\
Held-out examples & 2000 fixed samples \\
Evaluation frequency & Once per epoch \\
OOD test sizes & $n/2$ and $2n$ \\
Success criterion & Maximum held-out accuracy $1.0$ \\
Readout convention & RHI: query token, LHI: end token \\
Input query features & $(i-1)/n$ and $((i-1)/n)^2$ \\
Architectures &
\begin{tabular}[c]{@{}l@{}}
Full and causal softmax attention \\
Full and causal linear attention \\
GRU \\
Mamba2
\end{tabular} \\
Model width & 16 \\
Attention heads & 2 \\
Feed-forward width & 64 \\
Linear-attention feature map & $\operatorname{ELU}(x)+1$ \\
Dropout & 0 \\
Positional encoding & Polynomial \\
Loss & Binary cross-entropy with logits \\
Optimizer & AdamW \\
Learning rate & $10^{-3}$ \\
AdamW $\epsilon$ & $10^{-8}$ \\
AdamW betas & PyTorch default $(0.9,0.999)$ \\
Weight decay & 0 \\
Batch size & 256 \\
Gradient clipping & Global norm 5 \\
Early stopping & Held-out binary cross-entropy loss $\leq 10^{-6}$ \\
Scheduler & ReduceLROnPlateau \\
Scheduler metric & Training binary cross-entropy \\
Scheduler factor & 0.5 \\
Scheduler patience & 30 epochs \\
Min. learning rate & $10^{-5}$ \\
Mixed precision & No \\
Prediction rule & $\operatorname{sigmoid}(\text{logit})\geq 0.5$ \\
Max. epochs & 500 \\
\bottomrule
\end{tabular}
\caption{Final training and evaluation settings used in the main sweep.}
\label{tab:supp-experimental-setup}
\end{table}

Experiments were run on a Slurm-managed x86\_64 Linux server with kernel \texttt{6.8.0-124-generic}, two AMD EPYC 9654 96-core processors, 1.5 TiB of system memory, and NVIDIA H100 GPUs with 80 GB of HBM3 memory, using driver 610.43.02. The sweep launcher requested one GPU per run. The software environment used Python 3.12.3, PyTorch 2.12.1 with CUDA 13.0, Transformers 5.12.1, \texttt{flash-linear-attention} 0.5.1, Triton 3.7.1, and NumPy 2.5.0.

\newpage

\subsection{Positional-Encoding Ablation}

Before running the main sweep, we compared learned, sinusoidal, rotary, polynomial, and no positional encoding. This ablation uses $d=16$, two attention heads, $n\in\{32,64\}$, seeds $\{0,1,2\}$, and the same training protocol as the main sweep. The ablation covers nine architecture--depth configurations (rotary is evaluated only for the four softmax-attention configurations):

\begin{enumerate}
    \item RHI full and causal softmax attention with one layer
    \item RHI full linear attention with two layers
    \item LHI full and causal softmax and linear attention with two layers
    \item a one-layer LHI GRU
    \item two-layer LHI Mamba2
\end{enumerate}

For polynomial positional encoding, each token position $t\in\{0,\ldots,L-1\}$ is represented by $(p,p^2)$, where $p=t/(L-1)$ and $L$ is the training sequence length. This two-dimensional representation is projected to the model dimension using a learned bias-free linear layer.

The main paper reports only the conclusion of this ablation. Table~\ref{tab:supp-pe} presents the results by sequence length. Polynomial encoding is the only option that reaches perfect held-out accuracy in every run at $n=32$ and in all but one run at $n=64$. Based on these results, we use polynomial positional encoding throughout the main sweep.

\begin{table*}[!h]
\centering
\begin{tabular}{lrrrr}
\toprule
Encoding & $n$ & Runs & Successes & Mean max. accuracy \\
\midrule
  None & 32 & 27 & 12 & 0.7628 \\
  None & 64 & 27 & 12 & 0.7578 \\
  Learned & 32 & 27 & 17 & 0.9751 \\
  Learned & 64 & 27 & 8 & 0.8708 \\
  Sinusoidal & 32 & 27 & 23 & 0.9969 \\
  Sinusoidal & 64 & 27 & 18 & 0.9432 \\
  Rotary & 32 & 12 & 12 & 1.0000 \\
  Rotary & 64 & 12 & 6 & 0.9497 \\
  Polynomial & 32 & 27 & 27 & 1.0000 \\
  Polynomial & 64 & 27 & 26 & 0.9849 \\
\bottomrule
\end{tabular}
\caption{Positional-encoding ablation. Rotary is included only for softmax attention.}
\label{tab:supp-pe}
\end{table*}

The only polynomial run that does not reach perfect held-out accuracy is \texttt{L64\_softmaxL2d16h2\_poly\_s0}, with maximum held-out accuracy $0.5910$ after 500 epochs. All other polynomial runs in this ablation reach held-out accuracy $1.0$.

\newpage

\subsection{Complete Main-Sweep Accuracy}

Table~\ref{tab:supp-max-accuracy} complements the success counts reported in the main paper by showing the maximum held-out accuracy reached during training. Each entry reports the mean $\pm$ standard deviation across five seeds. This makes it possible to distinguish configurations that came close to the perfect-accuracy criterion of $1.0$ from those that remained substantially below perfect accuracy.

As in the previous sections, one-layer full and causal attention models are listed separately because both variants were trained independently.

\begin{table*}[!h]
\centering
\begin{tabular}{llclcccc}
\toprule
Where & Model family & Depth & Theory
& $n=8$ & $n=16$ & $n=32$ & $n=64$ \\
\midrule
RHI & Full softmax & 1 & Constructive
& $1.00\pm0.00$ & $1.00\pm0.00$ & $1.00\pm0.00$ & $1.00\pm0.00$ \\

RHI & Causal softmax & 1 & Constructive
& $1.00\pm0.00$ & $1.00\pm0.00$ & $1.00\pm0.00$ & $1.00\pm0.00$ \\

RHI & Full linear & 1 & Stress test
& $1.00\pm0.00$ & $1.00\pm0.00$ & $0.78\pm0.01$ & $0.64\pm0.01$ \\

RHI & Full linear & 2 & Constructive
& $1.00\pm0.00$ & $1.00\pm0.00$ & $1.00\pm0.00$ & $1.00\pm0.00$ \\

RHI & Causal linear & 1 & Stress test
  & $1.00\pm0.00$ & $1.00\pm0.00$ & $0.78\pm0.01$ & $0.65\pm0.01$ \\

RHI & Causal linear & 2 & Stress test
  & $1.00\pm0.00$ & $1.00\pm0.00$ & $0.99\pm0.02$ & $0.77\pm0.01$ \\

RHI & Causal linear & 3 & Stress test
& $1.00\pm0.00$ & $1.00\pm0.00$ & $0.99\pm0.01$ & $0.82\pm0.02$ \\

RHI & Causal linear & 4 & Stress test
  & $1.00\pm0.00$ & $1.00\pm0.00$ & $1.00\pm0.00$ & $0.85\pm0.04$ \\

RHI & GRU & 1 & Stress test
  & $1.00\pm0.00$ & $0.95\pm0.02$ & $0.77\pm0.02$ & $0.66\pm0.01$ \\

RHI & GRU & 2 & Stress test
  & $1.00\pm0.00$ & $1.00\pm0.00$ & $0.86\pm0.01$ & $0.69\pm0.01$ \\

RHI & GRU & 3 & Stress test
  & $1.00\pm0.00$ & $1.00\pm0.00$ & $0.86\pm0.02$ & $0.72\pm0.01$ \\

RHI & GRU & 4 & Stress test
  & $1.00\pm0.00$ & $1.00\pm0.00$ & $0.94\pm0.02$ & $0.74\pm0.01$ \\

RHI & Mamba & 1 & Stress test
  & $1.00\pm0.00$ & $0.99\pm0.01$ & $0.77\pm0.07$ & $0.66\pm0.02$ \\

RHI & Mamba & 2 & Stress test
  & $1.00\pm0.00$ & $1.00\pm0.00$ & $0.91\pm0.04$ & $0.69\pm0.04$ \\

RHI & Mamba & 3 & Stress test
& $1.00\pm0.00$ & $1.00\pm0.00$ & $0.97\pm0.03$ & $0.79\pm0.02$ \\

RHI & Mamba & 4 & Stress test
  & $1.00\pm0.00$ & $1.00\pm0.00$ & $0.99\pm0.02$ & $0.80\pm0.04$ \\

\midrule

LHI & GRU & 1 & Constructive
& $1.00\pm0.00$ & $1.00\pm0.00$ & $1.00\pm0.00$ & $1.00\pm0.00$ \\

LHI & Mamba & 1 & Stress test
  & $1.00\pm0.00$ & $0.82\pm0.06$ & $0.65\pm0.04$ & $0.59\pm0.01$ \\

LHI & Mamba & 2 & Constructive
& $1.00\pm0.00$ & $1.00\pm0.00$ & $1.00\pm0.00$ & $0.91\pm0.20$ \\

LHI & Full softmax & 1 & Stress test
  & $0.76\pm0.15$ & $0.64\pm0.06$ & $0.58\pm0.01$ & $0.56\pm0.02$ \\

LHI & Full softmax & 2 & Constructive
  & $1.00\pm0.00$ & $1.00\pm0.00$ & $1.00\pm0.00$ & $0.92\pm0.18$ \\

LHI & Causal softmax & 1 & Stress test
  & $0.76\pm0.15$ & $0.64\pm0.06$ & $0.58\pm0.01$ & $0.56\pm0.02$ \\

LHI & Causal softmax & 2 & Constructive
& $1.00\pm0.00$ & $1.00\pm0.00$ & $1.00\pm0.00$ & $1.00\pm0.00$ \\

LHI & Full linear & 1 & Stress test
  & $1.00\pm0.00$ & $0.78\pm0.15$ & $0.64\pm0.06$ & $0.58\pm0.02$ \\

LHI & Full linear & 2 & Constructive
& $1.00\pm0.00$ & $1.00\pm0.00$ & $1.00\pm0.00$ & $1.00\pm0.00$ \\

LHI & Causal linear & 1 & Stress test
  & $1.00\pm0.00$ & $0.77\pm0.14$ & $0.64\pm0.06$ & $0.58\pm0.02$ \\

LHI & Causal linear & 2 & Constructive
& $1.00\pm0.00$ & $1.00\pm0.00$ & $1.00\pm0.00$ & $1.00\pm0.00$ \\
\bottomrule
\end{tabular}

\caption{
Mean and standard deviation, across five seeds, of the maximum held-out
accuracy reached during training for each architecture--depth
configuration. \emph{Constructive} and \emph{Stress test} denote the
theory-guided groups defined in the main paper.
}
\label{tab:supp-max-accuracy}
\end{table*}

\newpage

\subsection{Runtime and Model Scale at $n=64$}

Table~\ref{tab:supp-cost-n64} reports parameter counts, training length, and wall-clock runtime for the constructive families at $n=64$. Successes are reported across five seeds. The number of parameters is fixed for each configuration, while the number of epochs and runtime measurements are reported as medians across the five runs.

Training time is measured for the final training epoch, and evaluation time corresponds to one pass over the fixed held-out set of 2,000 examples using the same batch size as training. These measurements document the computational scale of the experiments under the hardware and software setup described above.

\begin{table*}[!h]
\centering
\begin{tabular}{llcrrrrr}
\toprule
Where & Model & L & Succ. & Params & Med. epochs
& Train s/epoch & Eval s/pass \\
\midrule
RHI & Full softmax & 1 & 5/5 & 3473 & 185 & 2.21 & 0.065 \\
RHI & Causal softmax & 1 & 5/5 & 3473 & 185 & 2.24 & 0.065 \\
RHI & Full linear & 2 & 5/5 & 6577 & 73 & 2.55 & 0.068 \\
LHI & GRU & 1 & 5/5 & 1777 & 325 & 2.04 & 0.062 \\
LHI & Mamba & 2 & 4/5 & 3261 & 113 & 9.01 & 0.184 \\
LHI & Full softmax & 2 & 4/5 & 6753 & 261 & 2.57 & 0.069 \\
LHI & Causal softmax & 2 & 5/5 & 6753 & 112 & 2.45 & 0.069 \\
LHI & Full linear & 2 & 5/5 & 6577 & 50 & 3.24 & 0.085 \\
LHI & Causal linear & 2 & 5/5 & 6577 & 33 & 2.60 & 0.069 \\
\bottomrule
\end{tabular}
\caption{
Parameter counts, training length, and runtime at $n=64$ for the
constructive families. Epoch and runtime columns report medians
across five seeds.
}
\label{tab:supp-cost-n64}
\end{table*}

\newpage

\subsection{Length-Shifted Evaluation}

This section reports length-shifted evaluations for the constructive families. For each run, the model is evaluated on inputs with $n/2$ and $2n$ bits using the same parameters, tokenization, and readout convention, without retraining. These evaluations do not affect whether a run is counted as successful in-distribution.

At the task level, a shorter instance can always be embedded into a longer one by appending additional bits that are never queried. Thus, there is no fundamental obstacle to using the same indexing rule across different lengths. However, our models are trained at a single length, and the polynomial positional encoding uses the training sequence length as its reference scale. The learned solution may therefore rely on the positions seen during training rather than implement a length-independent indexing rule. Strong length generalization was therefore not expected, but we included this evaluation as a simple probe of whether it emerged nevertheless.

As in the complete main-sweep table, one-layer full and causal attention variants are counted separately here because they were trained independently. This gives 45 possible constructive-family runs at each training length. The tables report the length-shifted accuracy from the final training epoch and include only runs whose maximum in-distribution held-out accuracy reached $1.0$. Table~\ref{tab:supp-ood-by-n} aggregates these results by training length, while Table~\ref{tab:supp-ood-n64-family} reports them separately for each model family at $n=64$.

Transfer across lengths is weak overall. Accuracy at $n/2$ is above chance for models trained at the smallest length but approaches $0.5$ as the training length increases. Accuracy at $2n$ remains near chance throughout. At training length $n=64$, both length-shifted evaluations are close to chance for every listed family.

\begin{table}[!h]
\centering
\begin{tabular}{rccc}
\toprule
Train $n$ & ID succ. & Acc. at $n/2$ & Acc. at $2n$ \\
\midrule
8 & 45 & $0.629\pm 0.091$ & $0.535\pm 0.026$ \\
16 & 45 & $0.554\pm 0.042$ & $0.520\pm 0.018$ \\
32 & 45 & $0.527\pm 0.023$ & $0.510\pm 0.011$ \\
64 & 43 & $0.514\pm 0.014$ & $0.500\pm 0.007$ \\
\bottomrule
\end{tabular}
\caption{
Length-shifted accuracy for the constructive families, aggregated
by training length. Results are reported as mean $\pm$ standard
deviation across runs that succeeded in-distribution.
}
\label{tab:supp-ood-by-n}
\end{table}

\begin{table*}[!h]
\centering
\begin{tabular}{llcccc}
\toprule
Where & Model & L & ID succ. & Acc. at $n/2$ & Acc. at $2n$ \\
\midrule
RHI & Full softmax & 1 & 5/5 & $0.525\pm 0.008$ & $0.501\pm 0.007$ \\
RHI & Causal softmax & 1 & 5/5 & $0.525\pm 0.008$ & $0.501\pm 0.007$ \\
RHI & Full linear & 2 & 5/5 & $0.502\pm 0.005$ & $0.500\pm 0.008$ \\
LHI & GRU & 1 & 5/5 & $0.510\pm 0.011$ & $0.497\pm 0.006$ \\
LHI & Mamba & 2 & 4/5 & $0.513\pm 0.024$ & $0.499\pm 0.010$ \\
LHI & Full softmax & 2 & 4/5 & $0.523\pm 0.005$ & $0.504\pm 0.003$ \\
LHI & Causal softmax & 2 & 5/5 & $0.517\pm 0.010$ & $0.497\pm 0.007$ \\
LHI & Full linear & 2 & 5/5 & $0.497\pm 0.012$ & $0.503\pm 0.011$ \\
LHI & Causal linear & 2 & 5/5 & $0.513\pm 0.007$ & $0.499\pm 0.009$ \\
\bottomrule
\end{tabular}
\caption{
Length-shifted accuracy for the constructive families trained at
$n=64$. Results are reported as mean $\pm$ standard deviation across
runs that succeeded in-distribution.
}
\label{tab:supp-ood-n64-family}
\end{table*}

\end{document}